\documentclass[11pt]{article}

\usepackage{enumerate}
\usepackage{amssymb}
\usepackage{amsmath}
\usepackage{amsthm}
\usepackage{tikz}
\usetikzlibrary{automata}
\usepackage{tikz-cd}
\usepackage[colorlinks=true,linkcolor=blue,citecolor=blue]{hyperref}

\usepackage{nicematrix}

\usepackage{circledsteps}
\pgfkeys{/csteps/inner xsep=1.5pt}
\pgfkeys{/csteps/inner ysep=1.5pt}

\numberwithin{equation}{section}

\makeatletter
\def\th@plainsl{\slshape}
\makeatother

\theoremstyle{plainsl}
  \newtheorem{THM}{Theorem}[section]
  \newtheorem{LEM}[THM]{Lemma}
  \newtheorem{PROP}[THM]{Proposition}

\theoremstyle{definition}
  \newtheorem{DEF}[THM]{Definition}
  \newtheorem{EX}[THM]{Example}

  \newtheorem{ASSUMPTION}{Assumption}
  \newtheorem*{FACT}{Fact}

\renewcommand{\le}{\leqslant}
\renewcommand{\ge}{\geqslant}

\renewcommand{\phi}{\varphi}
\renewcommand{\epsilon}{\varepsilon}

\newcommand{\CC}{\mathbf{C}}

\newcommand{\DD}{\mathbf{D}}

\newcommand{\NN}{\mathbb{N}}

\newcommand{\RR}{\mathbb{R}}

\newcommand{\Boxed}[1]{\mbox{$#1$}}
\newcommand{\id}{\mathrm{id}}

\newcommand{\ID}{\mathrm{ID}}
\newcommand{\Ob}{\mathrm{Ob}}

\newcommand{\Mat}{\mathrm{Mat}}

\newcommand{\calO}{\mathcal{O}}

\newcommand{\calU}{\mathcal{U}}

\newcommand{\Set}{\mathbf{Set}}
\newcommand{\Vect}{\mathbf{Vect}}
\newcommand{\FdnVect}{\mathbf{FdnVect}}
\newcommand{\Top}{\mathbf{Top}}

\title{Coalgebras for categorical deep learning:\\Representability and universal approximation}
\author{Dragan Ma\v sulovi\'c\\
        University of Novi Sad, Faculty of Sciences\\
        Department of Mathematics and Informatics\\
        Trg Dositeja Obradovi\'ca 3, 21000 Novi Sad, Serbia\\
        email: dragan.masulovic@dmi.uns.ac.rs
}
\date{\today}

\begin{document}
\maketitle

\begin{abstract}
  Categorical deep learning (CDL) has recently emerged as a framework that leverages category 
  theory to unify diverse neural architectures. While geometric deep learning (GDL) is grounded
  in the specific context of invariants of group actions, CDL aims to provide domain-independent
  abstractions for reasoning about models and their properties. In this paper, we contribute to
  this program by developing a \emph{coalgebraic} foundation for equivariant representation
  in deep learning, as classical notions of group actions and equivariant maps are naturally
  generalized by the coalgebraic formalism. Our first main result demonstrates that,
  given an embedding of data sets formalized as a functor from $\Set$ to $\Vect$,
  and given a notion of invariant behavior on data sets modeled by an endofunctor on $\Set$,
  there is a corresponding endofunctor on $\Vect$ that is compatible with the embedding in the sense that
  this lifted functor recovers the analogous notion of invariant behavior on the embedded data.
  Building on this foundation, we then establish a universal approximation theorem for equivariant
  maps in this generalized setting. We show that continuous equivariant functions can be approximated
  within our coalgebraic framework for a broad class of symmetries. This work thus provides a
  categorical bridge between the abstract specification of invariant behavior and its concrete realization in neural architectures.

  \medskip
  
  \noindent
  \textbf{Key Words and Phrases:} categorical deep learning, coalgebraic modeling, equivariant representation, universal approximation
\end{abstract}

\section{Introduction}

Categorical deep learning (CDL) has recently emerged as a higher-level framework grounded in category theory.
CDL seeks to unify various approaches to deep learning by emphasizing compositionality and universal constructions as
core design principles for neural architectures \cite{CDL-position-paper} (see also~\cite{CDL-website} for extensive bibliography).
Unlike geometric deep learning (GDL), which is tightly bound to specific geometric formalisms
(in the sense of Felix Klein's Erlangen Program \cite{GDL-main,GDL-Protobook,GDL-website}),
CDL aims to provide domain-independent abstractions
for reasoning about models, learning dynamics, and structural invariants.
This categorical perspective has the potential to systematically organize existing architectures,
illuminate connections between them, and guide the design of new models with provable properties
(see for example~\cite{Cat-GBL-2022}).
In this sense, CDL represents a step toward a universal foundation for deep learning,
complementing and extending the insights gained from GDL.

While existing categorical approaches to deep learning have productively formalized the fundamental mechanisms of learning
such as backpropagation and gradient-based optimization \cite{backprop-as-ftr,gavran-2019},
in this paper we use the machinery of category theory to understand equivariant representation
of data at a very abstract level, and to show a corresponding universal approximation theorem (UAT) for equivariant maps.
Although our main motivation comes from the \emph{algebraic} theory of architectures introduced in \cite{CDL-position-paper}, 
we propose a model for categorical deep learning where equivariance is generalized using \emph{coalgebras}.

Coalgebraic modeling has emerged within theoretical computer science as a foundational framework for modeling dynamic and
state-based systems~\cite{rutten-2000}. While algebras are concerned with the composition of elements,
coalgebras capture the decomposition or observation of system behavior over time,
making them particularly suited for representing entities with evolving internal states.
This paradigm provides a uniform mathematical language for a wide array of computational phenomena,
including deterministic and non-deterministic automata, labeled transition systems, infinite data structures,
and probabilistic models. The key insight is that a broad class of systems can be formalized as coalgebras
for a suitable endofunctor, which specifies the system's transition type or observable behavior.
This categorical abstraction facilitates the development of generic theories and tools for key system properties,
including behavioral equivalence (often characterized by bisimulation), semantic logics, and minimization algorithms
that operate uniformly across different system types.

In Section~\ref{euat.sec.coalg} we review the basic notions of the theory of coalgebras and illustrate, in a series of examples,
how the coalgebraic language subsumes and generalizes the classical concepts of group actions and equivariant maps.

Our first general result, presented in Section~\ref{euat.sec.representability}, addresses the following situation.
Suppose we have an embedding of data sets into vector spaces, given by a functor $\Set \to \Vect$.
Suppose, further, that a particular type of invariant behavior on data sets can be modeled coalgebraically, that is,
as coalgebras for some endofunctor on $\Set$. We then show that there exists a ``compatible'' endofunctor on $\Vect$
that lifts this structure, allowing the same invariant behavior to be recovered on the embedded data.
Crucially, the construction follows from standard categorical techniques, requiring no additional ad hoc machinery.

Finally, in Section~\ref{euat.sec.UAT} we present a symmetrization-based approach to constructing equivariant approximations
building on the Universal Approximation Theorem for shallow neural networks.
Our general Universal Approximation Theorem (UAT) shows that general equivariant continuous functions can be approximated by
equivariant maps computable by feed-forward \emph{vector} neural networks with a single hidden layer.
Prior work (see, e.g.~\cite{DengEtAl-VectorNeurons2021,KatzirEtAl-VectorNeurons-2022})
has demonstrated that vector neural networks are well suited for capturing equivariance with
respect to specific symmetry groups such as $\mathit{SO}(3)$, thereby enabling models to respect the geometric structure of three-dimensional space.
While the symmetries arising in our coalgebraic framework are considerably more general, the core principle remains the same:
by leveraging the same class of architectures, we aim to extend these benefits to the class of equivariant
contexts modeled by coalgebras.

\section{Coalgebras}
\label{euat.sec.coalg}

We assume the familiarity with the basic notions of category theory. For technicalities we refer the reader to~\cite{AHS}
and~\cite{Riehl}. In particular, $\Set$ is the category of sets and functions between them, $\Vect$ is the category
of vector spaces and linear maps between them, and $\Top$ is the category of topological spaces and continuous maps
between them. We shall also make use of the category $\FdnVect$ of finitely-dimensional normed vector spaces and
\emph{continuous maps} between them. So, for each pair $V, W \in \Ob(\FdnVect)$ we have the set
$C(V, W) = \hom_{\FdnVect}(V, W)$ of continuous maps from $V$ to $W$, as well as the set $L(V, W) \subseteq C(V, W)$
of linear maps from $V$ to $W$.

A coalgebra is a structure that describes how a system evolves over time.
While algebras are modeled by morphisms $F(A) \to A$, coalgebras are modeled by morphisms of the form $A \to F(A)$.

Let $\CC$ be a category and $F : \CC \to \CC$ an endofunctor.
An \emph{$F$-coalgebra (in $\CC$)} is a pair $(A, \alpha)$ where $A \in \Ob(\CC)$ and $\alpha : A \to F(A)$ is a morphism in $\CC$.
The morphism $\alpha$ is the \emph{structure morphism} and the object $A$ is the \emph{carrier} of the coalgebra $(A, \alpha)$,
or the \emph{universe of $(A, \alpha)$} in case $\CC = \Set$.
The functor $F$ is called the \emph{type functor} of the coalgebra $(A, \alpha)$.

A morphism $f \in \hom_\CC(A, B)$ is a \emph{coalgebra homomorphism} from an $F$-coalgebra $(A, \alpha)$ to an $F$-coalgebra $(B, \beta)$ if:
\begin{center}
  \begin{tikzcd}
    A \arrow[r, "f"] \arrow[d, "\alpha"'] & B \arrow[d, "\beta"] \\
    F(A) \arrow[r, "F(f)"'] & F(B)
  \end{tikzcd}
\end{center}
The category whose objects are $F$-coalgebras in $\CC$ and morphisms are coalgebra homomorphisms is denoted by $\CC_F$.
In particular, for $F : \Set \to \Set$, by $\Set_F$ we denote the category of all $F$-coalgebras in $\Set$.

\begin{EX}
  Let $G$ be a group acting on a set $A$. That means that we are given a mapping $\xi : G \times A \to A$
  such that $\xi(1, a) = a$ and $\xi(g_1, \xi(g_2, a)) = \xi(g_1 g_2, a)$.
  Using the well-known adjunction $\hom(G \times A, A) \cong \hom(A, A^G)$ known as \emph{currying} in the
  functional programming community, instead of $\xi : G \times A \to A$ we can consider
  $\alpha : A \to A^G$ where $\alpha(a)$ is the \emph{input handler} $\alpha(a) = h_a \in A^G$
  defined by $h_a(g) = \xi(g, a)$. Therefore, the group action can be modeled by the \emph{algebra}
  $(A, \xi)$, but also by the $F$-coalgebra $(A, \alpha)$ for the $\Set$-endofunctor $F(X) = X^G$.

  Assume, now, that $G$ acts on sets $A$ and $B$ and let $\xi : G \times A \to A$ and
  $\theta : G \times B \to B$ be the group actions. Recall that a map $f : A \to B$ is an equivariant map
  if $\theta(g, f(a)) = f(\xi(g, a))$.

  Let $(A, \alpha)$ and $(B, \beta)$ be $F$-coalgebras that model the group actions $\xi$ and $\theta$, respectively.
  The main insight now (see also~\cite{CDL-position-paper}) is the following:
  \begin{quote}
    \it $f$ is an equivariant map between the two group actions if and only if $f$ is a coalgebra homomorphism
    from $(A, \alpha)$ to $(B, \beta)$.
  \end{quote}
  Namely, let $\alpha(a) = h_a$ be the input handler $h_a(g) = \xi(g, a)$
  and let $\beta(b) = k_b$ be the input handler $k_b(g) = \theta(g, b)$.
  Since $f$ is a coalgebra homomorphism, for a particular $a \in A$ we now have:
  \begin{center}
    \begin{tikzcd}
      a \arrow[r, mapsto, "f"] \arrow[d, mapsto, "\alpha"'] & f(a) \arrow[d, mapsto, "\beta"] \\
      h_a \arrow[r, mapsto, "{f \circ -}"'] & k_{f(a)}
    \end{tikzcd}
  \end{center}
  In other words, $f \circ h_a = k_{f(a)}$. If we instantiate this for an arbitrary $g \in G$ we get:
  $
    f(h_a(g)) = k_{f(a)}(g) \Rightarrow f(\xi(g, a)) = \theta(g, f(a))
  $.
\end{EX}

For this reason, we use the notions ``coalgebra homomorphism $f : (A, \alpha) \to (B, \beta)$'' and
``equivariant map $f : (A, \alpha) \to (B, \beta)$'' synonymously.

Let $F : \Set \to \Set$ be a $\Set$ endofunctor.
An $F$-coalgebra $(B, \beta)$ is a \emph{subcoalgebra} of an $F$-coalgebra $(A, \alpha)$
if $B \subseteq A$ and the inclusion map $\iota : B \to A : x \mapsto x$ is a coalgebra homomorphism:
\begin{center}
  \begin{tikzcd}
    B \arrow[r, "\iota"] \arrow[d, "\beta"'] & A \arrow[d, "\alpha"] \\
    F(B) \arrow[r, "F(\iota)"'] & F(A)
  \end{tikzcd}
\end{center}
If $(B, \beta)$ is a subcoalgebra of $(A, \alpha)$ we also say that the set $B$ is a \emph{subuniverse} of $(A, \alpha)$.
Subuniverses of coalgebras correspond to invariant subsets, as the following example suggests.

\begin{EX}
  Let $G$ be a group and $\xi : G \times A \to A$ an action of $G$ on a set $A$.
  Let $(A, \alpha)$ be the $F$-coalgebra that captures the action $\xi$, where
  $F : \Set \to \Set$ is the functor $F(A) = A^G$.
  Let $(S, \sigma)$ be a subcoalgebra of $(A, \alpha)$ and let us show that $S$ is an invariant subset of $A$
  in the sense that $\xi(g, x) \in S$ for all $g \in G$ and $x \in S$.
  Take any $x \in S$. Then $\sigma(x) \in F(S) = S^G$. In other words,
  the input handler $\sigma(x) : G \to S$ is a well-defined function,
  so for every $g \in G$ we have that $\sigma(x)(g) \in S$. If we translate this back to the language of group actions, we get
  $\xi(g, x) \in S$.
\end{EX}

It is well-known that for any category $\CC$ and any endofunctor $F : \CC \to \CC$, the forgetful functor
$U_F : \CC_F \to \CC : (A, \alpha) \mapsto A$ creates colimits (for a detailed exposition see~\cite{Bar93}).
In particular, $\Set_F$ and $\Vect_E$ are cocomplete for every $F : \Set \to \Set$ and $E : \Vect \to \Vect$.

A \emph{comonad} on a category $\CC$ is a triple $(E, \delta, \epsilon)$ where $E : \CC \to \CC$ is an endofunctor,
$\delta : E \Rightarrow EE$ is a natural transformation called the \emph{comultiplication} for $E$,
and $\epsilon : E \Rightarrow \ID$ is a natural transformation called the \emph{counit} for $\delta$.
The comultiplication $\delta$ is required to be \emph{coassociative}
in the sense that for every $A \in \Ob(\CC)$ the following diagram commutes:
\begin{center}
  \begin{tikzcd}
    E(A) \arrow[rr, "\delta_A"] \arrow[d, "\delta_A"'] & & EE(A) \arrow[d, "\delta_{E(A)}"] \\
    EE(A) \arrow[rr, "E(\delta_A)"'] & & EEE(A)
  \end{tikzcd}
\end{center}
The counit $\epsilon$ is required to be \emph{counital} in the sense that for every $A \in \Ob(\CC)$ the following diagram commutes:
\begin{center}
  \begin{tikzcd}
    E(A) \arrow[rr, "\delta_{A}"] \arrow[d, "\delta_A"'] \arrow[drr, "\id_{E(A)}" description] & & EE(A) \arrow[d, "\epsilon_{E(A)}"] \\
    EE(A) \arrow[rr, "E(\epsilon_A)"'] & & E(A)
  \end{tikzcd}
\end{center}

To each comonad $(E, \delta, \epsilon)$ we can straightforwardly assign the \emph{Eilenberg-Moore category}
whose objects are $E$-comodules (special $E$-coalgebras to be defined immediately),
morphisms are coalgebraic homomorphisms and the composition is as in $\CC$.
An \emph{$E$-comodule} is an $E$-coalgebra $\alpha : A \to E(A)$ for which the following two diagrams commute:
\begin{center}
  \begin{tabular}{c@{\qquad}c}
    \begin{tikzcd}
      A \arrow[r, "\alpha"] \arrow[d, "\alpha"']  & E(A) \arrow [d, "\delta_A"]\\
      E(A) \arrow[r, "E(\alpha)"'] & EE(A)
    \end{tikzcd}
    &
    \begin{tikzcd}
      A \arrow[d, "\alpha"'] \arrow[dr, "\id_A"] & \\
      E(A) \arrow[r, "\epsilon_A"'] & A
    \end{tikzcd}
  \end{tabular}
\end{center}

\begin{EX}
  Recall that for a vector space $V$ and a set $X$ by $V^X$ we denote the vector space of all
  functions $X \to V$ where the vector space structure is introduced
  coordinatewise:
  \begin{itemize}
    \item $(a \cdot \phi)(x) = a \cdot \phi(x)$ for all $x \in X$, and
    \item $(\phi + \psi)(x) = \phi(x) + \psi(x)$ for all $x \in X$,
  \end{itemize}
  where $\phi, \psi \in V^X$ and $a \in \RR$.

  Let $(G, \Boxed\cdot, 1)$ be a group. 
  Then the \emph{group action comonad (on vector spaces)} is the comonad $(E, \delta, \epsilon)$ defined as follows.
  Define $E : \Vect \to \Vect$ by $E(V) = V^G$ on objects, while for a mapping $f : V \to W$
  we let $E(f) : V^G \to W^G : h \mapsto f \circ h$.
  Next, define the comultiplication $\delta_V : E(V) \to EE(V)$ by
  $$
    \delta_V(\phi)(g)(h) = \phi(h \cdot g) \text{ for all }g, h \in G,
  $$
  and let the counit $\epsilon_V : E(V) \to V$ be given by
  $$
    \epsilon_V(\phi) = \phi(1).
  $$
  It is now a matter of straightforward calculation to verify that $(E, \delta, \epsilon)$ is
  a comonad on $\Vect$.
  
  Let us show that comodules for this comonad capture $G$-actions by linear maps.
  Let $\alpha : V \to V^G$ be an $E$-comodule. Define $\xi : G \times V \to V$ by
  $$
    \xi(g, v) = \alpha(v)(g)
  $$
  and let us show that this is a $G$-action, that is,
  $\xi(g, \xi(h, v)) = \xi(gh, v)$ and and $\xi(1, v) = v$ for all $g, h \in G$ and $v \in V$.
  The requirement
  \begin{center}
    \begin{tikzcd}
      V \arrow[d, "\alpha"'] \arrow[dr, "\id_V"] & \\
      E(V) \arrow[r, "\epsilon_V"'] & V
    \end{tikzcd}
  \end{center}
  tells us that $\epsilon_V \circ \alpha = \id_V$ so this applied to an arbitrary but fixed $v \in V$ gives: $\epsilon_V(\alpha(v)) = v$.
  By definition of $\epsilon$ this becomes $\alpha(v)(1) = v$, whence
  $$
    \xi(1, v) = v.
  $$
  The other requirement that $\alpha$ satisfies:
  \begin{center}
    \begin{tikzcd}
      V \arrow[r, "\alpha"] \arrow[d, "\alpha"']  & E(V) \arrow [d, "\delta_V"]\\
      E(V) \arrow[r, "E(\alpha)"'] & EE(V)
    \end{tikzcd}
  \end{center}
  means that $\delta_V \circ \alpha = E(\alpha) \circ \alpha$. Applying this to $v \in V$ gives:
  $$
    \delta_V (\alpha(v)) = E(\alpha)(\alpha(v)) \in EE(V) = (V^G)^G.
  $$
  Therefore, both $\delta_V (\alpha(v))$ and $E(\alpha)(\alpha(v))$ are functions $G \to(G \to V)$,
  so we have to verify that:
  \begin{equation}\label{coalg-cdl.eq.gacomonad-1}
    \delta_V (\alpha(v))(g)(h) = E(\alpha)(\alpha(v))(g)(h)
  \end{equation}
  for all $g, h \in G$. 
  The left-hand side evaluates straightforwardly by definition of $\delta$ and $\xi$:
  \begin{equation}\label{coalg-cdl.eq.gacomonad-2}
    \delta_V (\alpha(v))(g)(h) = \alpha(v)(hg) = \xi(hg, v).
  \end{equation}
  On the other hand, $\alpha(v)$ is a function $G \to V$, so, for the sake of simplicity,
  let $G = \{1, g_1, \ldots, g_n\}$ and
  $$
    \alpha(v) = \left(
    \begin{matrix}
      1 & g_1 & \ldots & g_n \\
      v & w_1 & \ldots & w_n
    \end{matrix}
  \right)
  $$
  for some $w_1, \ldots, w_n \in V$. (Note that we have already shown that $\alpha(v)(1) = v$.)
  Now, $E(\alpha)$ acts by unpacking the structure, applying $\alpha$ to the elements of $V$ and then
  repacking everything back. So,
  $$
    E(\alpha)(\alpha(v)) = \left(
      \begin{matrix}
        1 & g_1 & \ldots & g_n \\
        \alpha(v) & \alpha(w_1) & \ldots & \alpha(w_n)
      \end{matrix}
    \right)
  $$
  In particular, $E(\alpha)(\alpha(v))(g_1) = \alpha(w_1)$. But $w_1 = \alpha(v)(g_1)$, whence
  $E(\alpha)(\alpha(v))(g_1) = \alpha(\alpha(v)(g_1))$. This is, of course, true for every other element of $G$
  so the above description of $E(\alpha)(\alpha(v))$ can be written as follows:
  $$
    E(\alpha)(\alpha(v))(g) = \alpha(\alpha(v)(g)) \text{ for every } g \in G.
  $$
  Therefore, for every $g, h \in G$ we have:
  \begin{equation}\label{coalg-cdl.eq.gacomonad-3}
    E(\alpha)(\alpha(v))(g)(h)
    = \alpha(\alpha(v)(g))(h)
    = \alpha(\xi(g, v))(h)
    = \xi(h, \xi(g, v)).
  \end{equation}
  Putting \eqref{coalg-cdl.eq.gacomonad-1}, \eqref{coalg-cdl.eq.gacomonad-2} and \eqref{coalg-cdl.eq.gacomonad-3} together
  finally gives
  $$
    \xi(hg, v) = \xi(h, \xi(g, v)).
  $$
  So, $\xi$ is indeed a $G$-action.
\end{EX}

\section{Representability}
\label{euat.sec.representability}

Let $S$ be a set of samples and let $V(S)$ be a vector space which we think of as a feature space appropriate for $S$.
An \emph{embedding} is then a mapping $e : S \to V(S)$ that maps each sample to a point in the feature space.
For example, if $S$ is a set of images, then the embedding $e$ can be thought of as a \emph{feature extraction}
function that extracts features from the samples.

Assume now that $S$ is endowed with some structure, say, a group action $\xi : G \times S \to S$.
Then $S$ becomes a \emph{sample space} and the embedding $e$ is then a mapping from the sample space to the feature space.
The main idea behind this paper is that a sample space is modelled by a $\Set_F$-coalgebra
$$
  \alpha : S \to F(S)
$$
for an endofunctor $F : \Set \to \Set$. If we wish the embedding $e$ to be equivariant, the feature space should also be endowed with some structure
and then $e$ is expected to respect this structure.
However, we cannot expect the two structures to be modeled by the same functor $F$
since the feature space is a vector space and lives in $\Vect$ while the sample space is a set and lives in $\Set$.
So, one should be able to find an endofunctor $G : \Vect \to \Vect$ such that the corresponding feature space is modelled by a $G$-coalgebra
$$
  \overline\alpha : V(S) \to G(V(S))
$$
for some vector space $V(S)$ and a linear map $\overline \alpha$.
To implement equivariance we shall then have to ``repack'' the structure of the sample space to match the structure of the feature space.

Let us briefly recall the idea behind group representations in a setting that is equivalent to, yet slightly different from the standard one.
Let
$$
  \Mat(\RR) = \bigcup_{n \ge 1} \RR^{n \times n}
$$
denote the set of all real square matrices of all finite sizes (note that the matrix multiplication is
a partial operation on $\Mat(\RR)$). Then a linear representation of a group $G$ is any homomorphism
$$
  h : G \to \Mat(\RR).
$$
The representation is trivial if $h$ maps the entire group $G$ to a single matrix
which is then necessarily the identity matrix.
The requirement that $h$ be a homomorphism ensures that all the matrices $h(g)$, $g \in G$, have the same
size since
$$
  h(1 \cdot g) = h(1) \cdot h(g).
$$
This implicitly ensures that for every $g \in G$ the product $h(1) \cdot h(g)$ is defined,
whence $h(g)$ and $h(1)$ must be of the same size, for all $g \in G$.

All this can be succinctly described using the language of categories.
Let $[G]$ denote the category with a single object $*$ and with $\hom_{[G]}(*, *) = G$.
Then a linear representation of a group $G$ is a functor
$$
  H : [G] \to \Vect.
$$
The representation is trivial if $H$ is a constant functor.
This is our principal motivation behind the following notion:

\begin{DEF}
  Let $\CC$ be a category. A \emph{linear representation of $\CC$} is a functor $V : \CC \to \Vect$.
\end{DEF}

\begin{DEF}
  Let $F : \CC \to \CC$ be a functor. An \emph{equivariant representation of $\CC_F$} is a functor $V^* : \CC_F \to \Vect_E$
  for some functor $E : \Vect \to \Vect$. The representation is \emph{trivial} if $V^*$ is a constant functor; otherwise
  it is \emph{nontrivial}. The category $\CC_F$ \emph{has a (nontrivial) equivariant representation} if there exist a
  functor $E : \Vect \to \Vect$, and a (non-constant) functor $V^* : \CC_F \to \Vect_E$.
\end{DEF}

\begin{LEM}\label{euat.lem.lifting}
  Let $F : \CC \to \CC$ and $E : \DD \to \DD$ be endofunctors, and $V : \CC \to \DD$ a functor.
  If there exists a natural transformation $\lambda : VF \Rightarrow EV$ then $V$ can be lifted to
  a functor $V^* : \CC_F \to \DD_E$ such that:
  \begin{center}
    \begin{tikzcd}
      \CC_F \arrow[r, "V^*"] \arrow[d, "U_F"'] & \DD_E \arrow[d, "U_E"]\\
      \CC \arrow[r, "V"'] & \DD
    \end{tikzcd}
  \end{center}
  where $U_F : \CC_F \to \CC$ and $U_E : \DD_E \to \DD$ are the obvious forgetful functors that take a coalgebra $(A, \alpha)$
  onto its underlying object~$A$.
\end{LEM}
\begin{proof}
  Define $V^* : \CC_F \to \DD_E$ so that $V^*(A, \alpha) = (V(A), \lambda_A \cdot V(\alpha))$:
  $$
    V^*: \; \begin{array}{c}
      A \xrightarrow{\alpha} F(A) \\ \hline
      V(A) \xrightarrow{V(\alpha)} VF(A) \xrightarrow{\lambda_A} EV(A)
    \end{array}
  $$
  and $V^*(f) = V(f)$. Since $\lambda$ is natural, $V^*(f)$ is an $E$-coalgebra homomorphism for every coalgebra
  homomorphism $f : A \to B$:
  \begin{center}
  \begin{tikzcd}
    A \ar[d, "f"'] \ar[r, "\alpha"] & F(A) \ar[d, "F(f)"] & & V(A)  \ar[d, "V(f)"']   \ar[r, "V(\alpha)"] & VF(A) \ar[r, "\lambda_A"] \ar[d, "VF(f)"] & EV(A) \ar[d, "EV(f)"]\\
    B \ar[r, "\beta"']              & F(B)                & & V(B) \ar[r, "V(\beta)"']  & VF(B) \ar[r, "\lambda_B"'] & EV(B)
  \end{tikzcd}
  \end{center}
  Therefore, $V^* : \CC_F \to \DD_E$ is indeed a functor.
\end{proof}

\begin{DEF}
  A natural transformation $\lambda$ as in Lemma~\ref{euat.lem.lifting} will be referred to as a \emph{lifting}.
\end{DEF}

\begin{THM}\label{euat.thm.repr}
  Given a nontrivial linear representation $V : \Set \to \Vect$,
  for every endofunctor $F : \Set \to \Set$ there is an endofunctor $E : \Vect \to \Vect$
  and a nontrivial equivariant representation $V^* : \Set_F \to \Vect_E$.
\end{THM}
\begin{proof}
  Let $\calU$ be any family of sets such that $V$ restricted to $\calU$ is non-constant,
  and let $\Set^\calU$ be the full subcategory of $\Set$ with:
  $$
    \Ob(\Set^\calU) = \{F^n(A) : n \ge 0, A \in \calU\},
  $$
  where $F^n(A) = F( \ldots F(A))$ ($F$ is applied $n$ times), and $F^0(A) = A$.
  Let $V_\calU : \Set^\calU \to \Vect$ denote the domain restriction of $V$ to $\Set^\calU$,
  and let $F_\calU : \Set^\calU \to \Set^\calU$ denote the domain and codomain restriction of $F$ to $\Set^\calU$.
  Note that $\Set^\calU_{F_\calU}$ is the category of all $F$-coalgebras whose universes come from $\Set^\calU$, and hence
  $\Set^\calU_{F_\calU}$ is a full subcategory of $\Set_F$.

  Let $(E, \lambda)$ be the the left Kan extension of $V_\calU F_\calU$ along $V_\calU$ which exists because
  $\Set^\calU$ is small and $\Vect$ is cocomplete:
  \begin{center}
    \begin{tikzcd}[column sep=large, row sep=large]
                      & \Vect \ar[dr, dashed, bend left, "E"]  & \\
      \Set^\calU \ar[ur, bend left, "V_\calU"] \ar[r,"F_\calU"'] & \Set^\calU \ar[r, "V_\calU"'] \ar[u, Rightarrow, "\lambda"] & \Vect
    \end{tikzcd}
  \end{center}
  Note that $\lambda: V_\calU F_\calU \Rightarrow E V_\calU$ is a lifting, so Lemma~\ref{euat.lem.lifting} provides us with
  a construction of a functor $V^+ : \Set^\calU_{F_\calU} \to \Vect_{E}$. The construction ensures that
  $V^+$ is non-constant because $V_\calU$ is non-constant.

  Let $J : \Set^\calU_{F_\calU} \to \Set_F$ denote the inclusion functor (recall that $\Set^\calU_{F_\calU}$ is a full subcategory of $\Set_F$).
  Again, there is the left Kan extension $(V^*, \xi)$ of $V^+ : \Set^\calU_{F_\calU} \to \Vect_{E}$ along $J$
  because $\Set^\calU_{F_\calU}$ is small and $\Vect_E$ is cocomplete:
  
  \begin{center}
    \begin{tikzcd}[column sep=large, row sep=large]
       & \Set_F \arrow[dr, bend left, dashed, "V^*"]  \\
       \Set^\calU_{F_\calU} \arrow[rr, "V^+"'] \arrow[ur, bend left, "J"] & {\mathstrut} \arrow[u, Rightarrow, "\xi"] & \Vect_E
    \end{tikzcd}
  \end{center}
  Left Kan extensions along full and faithful functors behave as ``proper extensions'' in the sense that $\xi$
  is a natural isomorphism, so the fact that $V^+$ is a non-constant functor implies that $V^*$ is also non-constant.
\end{proof}

From the mere existence of a functor $V^* :\Set_F \to \Vect_E$ it is not obvious in what way
a particular coalgebra $(A, \alpha)$ in $\Set_F$ is \emph{actually} represented by the coalgebra $V^*(A, \alpha)$.
With a slightly more elaborate infrastructure, we can construct an explicit \emph{embedding} of
each coalgebra $(A, \alpha) \in \Ob(\Set_F)$. The embedding cannot go directly into the corresponding
coalgebra $V^*(A, \alpha)$ because $(A, \alpha)$ and $V^*(A, \alpha)$ reside in different categories.
Instead, the embedding goes from $(A, \alpha)$ into $U^*V^*(A, \alpha)$, where $U^*$ acts as a forgetful
functor $\Vect_E \to \Set_F$. Intuitively, there is no meaningful way to map a set with no structure
into a vector space as an algebraic structure,
but there is a canonical way to construct a function that embeds the set of samples into the set of vectors,
uniformly for all sample spaces. The equivariant nature of the embedding is captured by the fact that this is a
$\Set_F$-coalgebra homomorphism.

The intuition behind Proposition~\ref{euat.prop.inv-embedd} below is the following.
Assume that we can encode data sets as vectors via a functor $V : \Set \to \Vect$, and that we have somehow
identified two types of invariant behavior captured by functors $F : \Set \to \Set$ and $E : \Vect \to \Vect$.
When certain technical requirements are met, we can lift the representation $V : \Set \to \Vect$
to the equivariant representation $V^* : \Set_F \to \Vect_E$ which now takes into account the invariant
behaviors encoded by $F$ and $E$. Moreover, from the construction we can also extract the actual equivariant
embeddings $\eta^*_{(A, \alpha)} : (A, \alpha) \to U^* V^*(A, \alpha)$, and these embeddings are constructed
uniformly, for all $(A, \alpha) \in \Ob(\Set_F)$.

\begin{PROP}\label{euat.prop.inv-embedd}
    Given endofunctors $F : \Set \to \Set$ and $E : \Vect \to \Vect$,
    assume that there exists a linear representation $V : \Set \to \Vect$, a functor $U : \Vect \to \Set$
    and a natural transformation $\eta : \ID_{\Set} \Rightarrow UV$.
    If $\lambda : VF \Rightarrow EV$ and $\kappa : UE \Rightarrow FU$ are liftings satisfying:
    $$
      F\eta = \kappa V \circ U \lambda \circ \eta F,
    $$
    then $V$ and $U$ lift to functors $V^* : \Set_F \to \Vect_E$ and $U^* : \Vect_E \to \Set_F$, respectively,
    and $\eta$ lifts to a natural transformation $\eta^* : \ID_{\Set_F} \Rightarrow U^* V^*$ such that
    $\eta^*_{(A, \alpha)}$ is an equivariant map of $(A, \alpha)$ into $U^* V^*(A, \alpha)$
    for every $(A, \alpha) \in \Ob(\Set_F)$.
\end{PROP}
\begin{proof}
  Let us first note that $F\eta = \kappa V \circ U \lambda \circ \eta F$ means that for each set $A$:
  \begin{center}
      \begin{tikzcd}
        F(A) \arrow[r, "\eta_{F(A)}"] \arrow[d, equals] & UVF(A) \arrow[d, "\kappa_{V(A)} \circ U(\lambda_A)"] \\
        F(A) \arrow[r, "F(\eta_{A})"] & FUV(X)
      \end{tikzcd}
  \end{center}
  Lift the functors $V$ and $U$ to functors $V^*$ and $U^*$ as in Lemma~\ref{euat.lem.lifting}:
  \begin{align*}
    V^*(A, \alpha) &= (V(A), \lambda_A \circ V(\alpha)), \quad V^*(f) = V(f), \quad\text{and}\\
    U^*(X, \beta)  &= (U(X), \kappa_X  \circ U(\beta)), \quad U^*(\ell) = U(\ell).
  \end{align*}
  Then
  $$
    U^* V^*(A, \alpha) = (UV(A), \kappa_{V(A)} \circ U(\lambda_A) \circ UV(\alpha)),
  $$
  and $\eta_A$ is an equivariant map of $(A, \alpha)$ into $U^* V^*(A, \alpha)$ since:
  \begin{center}
    \begin{tikzcd}
        A \arrow[d, "\eta_A"'] \arrow[r, "\alpha"'] & F(A) \arrow[rr, equals] \arrow[d, "\eta_{F(A)}"] & & F(A) \arrow[d, "F(\eta_A)"] \\
        UV(A) \arrow[r, "UV(\alpha)"'] & UVF(A) \arrow[r,"U(\lambda_A)"'] & UGV(A) \arrow[r, "\kappa_{V(A)}"'] & FUV(A)
    \end{tikzcd}
  \end{center}
  So, define $\eta^* : \ID_{\Set_F} \Rightarrow U^* V^*$ by $\eta^*_{(A, \alpha)} = \eta_A$. The naturality of $\eta^*$
  now follows immediately from the fact that $\eta$ is natural, $\eta^*_{(A, \alpha)} = \eta_A$ and $U^*V^*(f) = UV(f)$:
  \begin{center}
    \begin{tikzcd}
      (A, \alpha) \arrow[r, "\eta^*_{(A,\alpha)}"] \arrow[d, "f"'] & U^* V^*(A, \alpha) \arrow[d, "U^*V^*(f)"] \\
      (B, \beta) \arrow[r, "\eta^*_{(B,\beta)}"'] & U^* V^*(B, \beta)
    \end{tikzcd}
    \hbox{\qquad$\leadsto$\qquad}
    \begin{tikzcd}
      A \arrow[r, "\eta_A"] \arrow[d, "f"'] & UV(A) \arrow[d, "UV(f)"]\\
      B \arrow[r, "\eta_B"'] & UV(B)
    \end{tikzcd}
  \end{center}
  This concludes the proof.
\end{proof}

Although the requirements of the above proposition may seem overwhelming, it is not hard to meet them
since transformations $\eta$, $\kappa$ and $\lambda$ capture no actual computing:
each of them just is repacks a data structure of one type into an essentially the same data structure of some other type
(think of repacking lists into vectors). The following example illustrates this point.

\begin{EX}
  Let $V : \Set \to \Vect$ be the functor that assigns to each set $A$ the freely generated vector space $V(A)$, and let
  $U : \Vect \to \Set$ be the forgetful functor. Since $(V, U)$ is an adjoint pair, there is a natural transformation
  $\eta : \ID \Rightarrow UV$. In this case, $\eta_A : A \to UV(A)$ is just the inclusion (since it is reasonable to assume that $A \subseteq UV(A)$).
  
  Fix a finite set $X = \{x_1, \ldots, x_n\}$ and let $F(A) = A^X$ in $\Set$ and $E(V) = V^X$ in $\Vect$. Then for a vector space $W$,
  $\kappa_W$ has to take $UE(W) = W^X$ naturally to $FU(W) = W^X$,
  so the most natural choice for $\kappa_W : W^X \to W^X$ is the identity mapping.
  On the other hand, $\lambda_A$ has to take $VF(A) = V(A^X)$
  naturally to $EV(A) = V(A)^X$, and this is also easy. Recall that $V(A^X)$ is a vector space freely generated by functions from $A^X$,
  so a typical element of $V(A^X)$ is a formal sum:
  $$
    c_1 \phi_1 + \ldots + c_k \phi_k
  $$
  for some $c_1, \ldots, c_k \in \RR$ and $\phi_1, \ldots, \phi_k \in A^X$. With this in mind, let
  $$
    \lambda_A(c_1 \phi_1 + \ldots + c_k \phi_k) = \psi \in V(A)^X
  $$
  where
  $$
    \psi(x) = c_1 \phi_1(x) + \ldots + c_k \phi_k(x) \in V(A).
  $$
  In particular, $\lambda_A(\phi) = \phi$.

  Finally, let us check that
  $$
    F(\eta_A) = \kappa_{V(A)} \circ U(\lambda_A) \circ \eta_{F(A)}
  $$
  holds for every set $A$. On the left-and side we have $F(\eta_A) : A^X \to (UV(A))^X$. So, for an arbitrary
  $\phi = \left(\begin{smallmatrix}
    x_1 & \ldots & x_n\\
    a_1 & \ldots & a_n
  \end{smallmatrix}\right)\in A^X$ we have that
  $$
    F(\eta_A)(\phi) = 
    \left(\begin{smallmatrix}
      x_1 & \ldots & x_n\\
      \eta_A(a_1) & \ldots & \eta_A(a_n)
    \end{smallmatrix}\right) = 
    \left(\begin{smallmatrix}
      x_1 & \ldots & x_n\\
      a_1 & \ldots & a_n
    \end{smallmatrix}\right) = \phi 
  $$
  since $\eta_A(a) = a$ for all $a \in A \subseteq UV(A)$. By the same argument, $\eta_{F(A)}(\phi) = \phi$. Recall, also,
  that $\kappa$ is the identity and that $U$ is the forgetful functor whence $U(\lambda_A) = \lambda_A$. Therefore, the right-hand side
  collapses to
  $$
    \kappa_{V(A)} \circ U(\lambda_A) \circ \eta_{F(A)}(\phi) = \lambda_A(\phi).
  $$
  But we have already seen that $\lambda_A(\phi) = \phi$.
\end{EX}

The above results are all based on the assumption that the feature extraction procedure is a functor $V : \CC \to \Vect$,
implying that we have a uniform way of extracting features from all sample sets in $\CC$.
\emph{Functorial feature extraction} seems to be a natural assumption in machine learning since
it is reasonable to expect that the feature extraction process should be uniform across all sample sets.
The results of this section show that if we have a nontrivial functorial feature extraction,
then we can always find a nontrivial equivariant representation for the coalgebraic model of
sample spaces.

\section{The Universal Approximation Theorem for the Coalgebraic Model}
\label{euat.sec.UAT}

In this section we present a symmetrization–based approach to constructing equivariant approximations
relying on the Universal Approximation Theorem for shallow neural networks.
The Universal Approximation Theorem (UAT) shows that continuous functions can be approximated by
feed-forward neural networks with a single hidden layer if the hidden layer is allowed to be arbitrarily wide.

\begin{THM}\label{euat.thm.UAT-main} \cite[Theorem 3.1]{Pinkus-1999}
  Let $\sigma : \RR \to \RR$ be a non-polynomial continuous activation function.
  Then any continuous map $f : \RR^n \to \RR^m$ can be approximated, in the sense of uniform convergence
  on compact sets, by maps of the form $Q \circ \bar\sigma \circ P$, where $P : \RR^n \to \RR^d$
  and $Q : \RR^d \to \RR^m$ are linear transformations, $d \in \NN$ is a positive integer and
  $\bar\sigma(x_1, \ldots, x_d) = (\sigma(x_1), \ldots, \sigma(x_d))$.
\end{THM}

\begin{DEF}
  Let $\mathcal{NN}^{\sigma}$ denote the set of all the functions of the form
  $$
    \RR^n \overset{P}\longrightarrow \RR^d \overset{\bar\sigma}\longrightarrow \RR^d \overset{Q}\longrightarrow \RR^m
  $$
  where $P$, $Q$ and $\bar\sigma$ are as in the formulation of the Theorem~\ref{euat.thm.UAT-main}
  and $n, m, d \in \NN$ are arbitrary positive integers.
\end{DEF}

Let us briefly outline the abstract context in which we shall prove our UAT.
Every UAT is about the interplay of continuous maps and maps constructed from linear maps between finitely-dimensional spaces
and a nonlinear activation function. Since the idea of a category is to capture one mathematical context, what we need is the interplay of
two categories: one capturing the context of finitely-dimensional normed vector spaces and linear maps,
and the other capturing the context of topological spaces, continuous maps and compact sets.

The primary context in this section is the category $\FdnVect$ of finitely-dimensional real normed vector spaces
and \emph{continuous} maps between them. Without loss of generality we may assume that this is a small category since
the normed vector spaces of interest are of the form $(\RR^n, \|\cdot\|)$, $n \in \NN$.
Instead of $\hom_{\FdnVect}(V, W)$ we shall simply write $C(V, W)$. By $L(V, W)$ we shall
denote the linear maps between $V$ and $W$.
Note that $L(V, W) \subseteq C(V, W)$ for all finitely-dimensional real normed vector spaces $V$ and $W$.

The category $\Top$ of topological spaces and continuous maps between them will provide the topological context.
Note that there is the usual forgetful functor $U : \FdnVect \to \Top$ that takes
a normed vector space $V = (V, \Boxed{+}, \|\cdot\|)$
to the topological space $U(V) = (V, \calO_{\|\cdot\|})$, where $\calO_{\|\cdot\|}$ is the topology
induced by the norm $\|\cdot\|$. The forgetful functor $U$ takes a continuous map $f : V \to W$
between normed vector spaces $V$ and $W$ to the continuous map $U(f) : U(V) \to U(W)$.
The topology of $U(V)$ is determined by the norm of~$V$, so we have the following important:

\begin{FACT}
  The functor $U : \FdnVect \to \Top$ is full and faithful.
\end{FACT}

Topological spaces and normed vector spaces are structured sets, so
morphisms between them are $\Set$-functions (with additional properties).
We may, therefore, safely assume that $U(f) = f$. Consequently,
if $V$ and $W$ are finitely-dimensional normed vector spaces we can talk about linear maps between
the topological spaces $U(V)$ and $U(W)$.
To keep the notation manageable we shall overload symbols $L$ and $C$ so that
$C(X, Y) = \hom_{\Top}(X, Y)$ for topological spaces $X$ and $Y$, and
$L(U(V), U(W)) = L(V, W)$ for normed vector spaces $V$ and~$W$:
$$
  \begin{array}{c@{\;}c@{\;}c@{\;}c@{\;}c@{\;}c@{\;}c}
    L(U(V), U(W)) &=& L(V, W) &\subseteq& C(V, W) &=& C(U(V), U(W)).\\
    \scriptstyle\uparrow & & \scriptstyle\uparrow & & \scriptstyle\uparrow & & \scriptstyle\uparrow\\
    \scriptstyle\Top & & \scriptstyle\FdnVect & & \scriptstyle\FdnVect & & \scriptstyle\Top
  \end{array}
$$
The left equality holds because we defined
$L(U(V), U(W))$ in $\Top$ to be $L(V, W)$, while the right equality follows from the fact that
the topology of $U(V)$ is uniquely determined by the norm of~$V$.

Most of the constructions in this section take place in $\FdnVect$.
We need to move to $\Top$ just to handle equivariant embeddings
of sample spaces (which are usually finite and, hence, do not live in $\FdnVect$), and
to consider spaces of continuous functions with the compact domain.

\begin{ASSUMPTION}
  For the coalgebraic modeling of equivariance we need an endofunctor $E : \FdnVect \to \FdnVect$.
  We assume that there is a natural transformation $\delta : E \Rightarrow EE$ that captures the
  internal dynamics of the functor (note that $\delta$ is not required to have additional properties;
  in particular we do \emph{not} assume that $\delta$ is a comultiplication).
\end{ASSUMPTION}

Since $\FdnVect$ is small and $\Top$ is cocomplete, there is a left Kan extension
$(E', \kappa)$ of $UE$ along $U$:
  \begin{center}
    \begin{tikzcd}[column sep=large, row sep=large]
                      & \Top \ar[dr, dashed, bend left, "E'"]  & \\
      \FdnVect \ar[ur, bend left, "U"] \ar[r,"E"'] & \FdnVect \ar[r, "U"'] \ar[u, Rightarrow, "\kappa"] & \Top
    \end{tikzcd}
  \end{center}
Note that $\kappa : UE \Rightarrow E'U$ is a natural isomorphism because $U$ is full and faithful.
Therefore, without loss of generality we may assume that
$\kappa$ is the identity and that $UE = E'U$. For this reason we shall simply write $E$ instead of~$E'$. Thus,
\begin{center}
  \begin{tikzcd}
    \FdnVect \arrow[r, "E"] \arrow[d, "U"'] & \FdnVect \arrow[d, "U"] \\
    \Top   \arrow[r, "E"']                  & \Top
  \end{tikzcd}
\end{center}
Consequently, given normed vector spaces $V$ and $W$, put $X = U(V)$ and $Y = U(W)$ and note that
$$
  E_{X, Y}(L(X, Y)) \subseteq L(E(X), E(Y)).
$$
In short, $E : \Top \to \Top$ takes linear maps to linear maps.

If $X = U(V)$ and $Y = U(W)$ for some normed real vector spaces $V = (V, \Boxed+, \|\cdot\|_V)$ and
$W = (W, \Boxed+, \|\cdot\|_W)$ then $C(X, Y)$ is a real vector space. Namely, $C(X, Y)$ is the set of all maps
$f : X \to Y$ which are continuous with respect to topologies induced by norms $\|\cdot\|_V$ and $\|\cdot\|_W$.
It is a well-known fact that if $f, g : X \to Y$ are continuous with respect to topologies induced by
a pair of norms (one on $X$ and the other one on $Y$),
then for all $c, d \in \RR$ we have that $cf + dg$ is also continuous with respect to topologies induced by
the same pair of norms. Therefore, if $f, g \in C(X, Y)$ and $c, d \in \RR$ then $cf + dg \in C(X, Y)$.

\begin{ASSUMPTION}
  We shall also assume that if $X = U(V)$ and $Y = U(W)$ for normed real vector spaces
  $V$ and $W$, then $E_{X, Y} : C(X, Y) \to C(E(X), E(Y))$ is a continuous linear transformation.
\end{ASSUMPTION}

\begin{DEF}
  We say that an $E$-coalgebra $\alpha : W \to E(W)$ in $\FdnVect$ is an \emph{$(E, \delta)$-comodule} if
    \begin{center}
      \begin{tikzcd}
        W \arrow[r, "\alpha"] \arrow[d, "\alpha"']  & E(W) \arrow [d, "\delta_W"]\\
        E(W) \arrow[r, "E(\alpha)"'] & EE(W)
      \end{tikzcd}
    \end{center}
  (compare with the definition of an $E$-comodule for a comonad).

  We shall say that an $E$-algebra $\gamma : E(W) \to W$ is a \emph{left $(E, \delta)$-inverse} of an $E$-coalgebra
  $\alpha : W \to E(W)$ in $\FdnVect$ if $\gamma \circ \alpha = \id_W$ and $\alpha \circ \gamma = E(\gamma) \circ \delta_W$:
    \begin{center}
      \begin{tikzcd}
        E(W) \arrow[r, "\gamma"] \arrow[d, "\delta_W"']  & W \arrow [d, "\alpha"]\\
        EE(W) \arrow[r, "E(\gamma)"'] & E(W)
      \end{tikzcd}
    \end{center}
  We say that $\alpha$ is \emph{left $(E,\delta)$-invertible} if such a $\gamma$ exists.
\end{DEF}

\begin{EX} (compare with \cite[Propositions 2.1, 2.2]{Yarotsky-2022}).
  Let $G$ be a finite group that acts on finitely dimensional real normed spaces $V$ and $W$ by linear transformations,
  and let $gx$, resp.\ $gy$, denote the action of $g \in G$ on $x \in V$, resp.\ $y \in W$.
  Let $\phi : V \to W$ be a $G$-equivariant continuous map, that is, a continuous map satisfying:
  $$
    \phi(gx) = g \phi(x), \quad \text{for all } g \in G.
  $$
  By the Universal Approximation Theorem, for every compact $K \subseteq V$ and $\epsilon > 0$
  there is a ``nice'' (say, piecewise linear) function $f : V \to W$ that approximates $\phi$ on $K$ up to $\epsilon$:
  $$
    \|\phi(x) - f(x)\| < \epsilon \quad \text{for all } x \in K.
  $$
  However, $f$ is not necessarily $G$-equivariant. The issue can be resolved by symmetrizing $f$ and $K$. Let
  $$
    \widehat K = \bigcup_{g \in G} gK
  $$
  and for $x \in \widehat K$ let:
  $$
    \widehat f(x) = \frac{1}{|G|}\sum_{g \in G} g^{-1}f(gx)
  $$
  One can now check that $\widehat f$ is $G$-equivariant and approximates $\phi$ on $\widehat K$ up to $C \cdot \epsilon$
  where $C \in \RR$ depends only on $G$ (see~\cite[Propositions 2.1, 2.2]{Yarotsky-2022} for details).
  
  Note that $\Phi : C(V, W) \to C(V, W)$ given by $\Phi(f) = \widehat f$ is a continuous linear operator on $C(V, W)$
  such that every $G$-equivariant map is a fixed point of~$\Phi$.
  
  Let us fit this example into the abstract context outlined at the beginning of the section.
  
  Let $G$ be a finite group and $E : \FdnVect \to \FdnVect$ a functor given by $E(V) = V^G$ on objects and
  $E(f) = f \circ -$ on morphisms. The compatible functor $E : \Top \to \Top$ on the category of
  topological spaces is the same: $E(V) = V^G$ on objects and $E(f) = f \circ -$ on morphisms.
  The consequence of the compatibility requirement is the fact that
  $E_{V, W}(L(V, W)) \subseteq L(V^G, W^G)$. Moreover,
  $E_{V, W} : C(V, W) \to C(V^G, W^G)$ is a continuous linear transformation.
  
  Let $\delta : E \Rightarrow EE$ be given by $\delta_V : V^G \to (V^G)^G$ were $\delta_V(\phi)(g)(h) = \phi(hg)$. 

  Let $\beta : V \to E(V)$ be an $E$-coalgebra that captures the action of $G$ on $V$ by linear maps. This means that
  $\beta(x) = \beta_x \in V^G$ where $\beta_x(g) = gx$. Define an $E$-algebra $\gamma : E(V) \to V$ as follows:
  for $\phi \in E(V) = V^G$ we let 
  $$
    \gamma(\phi) = \frac{1}{|G|} \sum_{g \in G} g^{-1} \phi(g).
  $$
  Let us show that $\gamma$ is a left $(E, \delta)$-inverse of $\beta$.
  To show that $\gamma \circ \beta = \id_V$, take any $x \in V$ and compute:
  \begin{align*}
    \gamma \circ \beta(x)
      &= \gamma(\beta(x)) = \gamma(\beta_x) = \\
      &= \dfrac{1}{|G|} \sum_{g \in G} g^{-1} \beta_x(g) = \dfrac{1}{|G|} \sum_{g \in G} g^{-1} gx = \dfrac{1}{|G|} \sum_{g \in G} x = x.
  \end{align*}
  
  To show that $\beta \circ \gamma = E(\gamma) \circ \delta_V$, note, first, that $\beta \circ \gamma : V^G \to V^G$. So, take any $\phi \in V^G$. Then
  $$
    \beta \circ \gamma(\phi) = \beta(\gamma(\phi)) = \beta\bigg(\frac{1}{|G|} \sum_{g \in G} g^{-1} \phi(g)\bigg) = \beta_s,
  $$
  for
  $
    s = \frac{1}{|G|} \sum_{g \in G} g^{-1} \phi(g)
  $.
  So,
  $$
    \beta(\gamma(\phi))(h) = \beta_s(h) = hs = \frac{1}{|G|} \sum_{g \in G} hg^{-1} \phi(g).
  $$
  On the other hand, $E(\gamma) \circ \delta_V (\phi) = \gamma \circ \delta_V(\phi) \in V^G$. So, take any $h \in G$.
  Then (using the $\lambda$-notation for $\delta_V(\phi)(h) = \lambda g.\phi(gh)$):
  $$
    (\gamma \circ \delta_V(\phi))(h)
    = \gamma(\delta_V(\phi)(h)) = \gamma(\lambda g.\phi(gh))
    = \frac{1}{|G|} \sum_{g \in G} g^{-1} \phi(gh).
  $$
  Let us reindex the sum so that $k = gh$. Then $g^{-1} = hk^{-1}$ and the sum becomes:
  $$
    (\gamma \circ \delta(\phi))(h) = \frac{1}{|G|} \sum_{k \in G} hk^{-1} \phi(k).
  $$
  So, $\gamma$ is indeed a left $(E, \delta)$-inverse of $\beta$.
  Clearly, $\gamma$ is a linear transformation, and this observation concludes the example.
\end{EX}

\paragraph{Vector neural networks.}
It has recently been recognized that organizing a neural network around
vector neurons (instead of ``scalar neurons'') is beneficial for an efficient implementation of equivariance
in the context of deep neural networks aimed at computer vision
(see e.g.~\cite{DengEtAl-VectorNeurons2021,KatzirEtAl-VectorNeurons-2022}).

A \emph{vector neural network} is a feed-forward neural network where each neuron is a \emph{vector neuron}
modeled by a $k$-tuple of scalars for some $k \in \NN$, and the activation function is a continuous function
$\rho : \RR^k \to \RR^k$ defined on $k$-tuples (not necessarily coordinatewise).
More precisely, the behavior of a a vector neuron $\mathbf{y} \in \RR^k$ is modeled as:
$$
  \mathbf{y} = \rho\left(\sum_{i=1}^p A_i \mathbf{x}_i\right),
$$
where $\mathbf{x}_1, \ldots, \mathbf{x}_p \in \RR^n$ are
vector neurons in the previous layer, $A_1, \ldots, A_p \in \RR^{k \times n}$ are weights now implemented as matrices,
and $\rho : \RR^k \to \RR^k$ is the activation function.
The additional flexibility of vector neural networks comes from the fact that
the activation of a vector neuron now depends on the entire vector of scalars,
in contrast to the standard feed-forward neural networks where the activation function is
applied scalar by scalar.

\begin{DEF}
  Let $\rho : \RR^k \to \RR^k$ be a continuous activation function.
  By $\mathcal{VNN}^{\rho}$ we denote the set of all the functions of the form
  $$
    \RR^n \overset{P}\longrightarrow (\RR^k)^d \overset{\bar\rho}\longrightarrow (\RR^k)^d \overset{Q}\longrightarrow \RR^m
  $$
  where $n, m, d, k \in \NN$ are positive integers,
  $P : \RR^n \to (\RR^k)^d$ and $Q : (\RR^{k})^d \to \RR^m$ are linear transformations, and
  $$
    \bar\rho(\mathbf{x}_1, \ldots, \mathbf{x}_d) = (\rho(\mathbf{x}_1), \ldots, \rho(\mathbf{x}_d)),
  $$
  for all $\mathbf{x}_1, \ldots, \mathbf{x}_d \in \RR^k$.
  We shall say that functions in $\mathcal{VNN}^{\rho}$ are \emph{$\mathcal{VNN}^{\rho}$-computable}.
\end{DEF}

\begin{THM}
  Let $\sigma : \RR \to \RR$ be a non-polynomial continuous activation function.
  Let $V$ and $W$ be finitely-dimensional normed vector spaces and let
  $\alpha : V \to E(V)$ and $\beta : W \to E(W)$ be linear maps such that
  $\alpha$ is an $(E, \delta)$-module and $\beta$ has a left $(E, \delta)$-inverse which is a linear transformation.

  Then, for every continuous equivariant $\phi : (V, \alpha) \to (W, \beta)$, every subcoalgebra $(K, \theta)$ of
  $(U(V), \alpha)$ in $\Top$ where $K \subseteq U(V)$ is compact, and every $\epsilon > 0$
  there exists an $\mathcal{VNN}^{E(\sigma)}$-computable equivariant $\ell : (V, \alpha) \to (W, \beta)$ such that
  $\|\phi(x) - \ell(x)\| < \epsilon$ for all $x \in K$.
\end{THM}
\begin{proof}
  Let $\gamma : E(W) \to W$ be a linear transformation which is a left $(E, \delta)$-inverse of $\beta$.
  For $f \in C(V, W)$ let 
  $$
    \Phi(f) = \gamma \circ E(f) \circ \alpha.
  $$
  Note that $\Phi$ is a continuous linear operator $C(V, W) \to C(V, W)$.

  \bigskip

  Claim~1. For every continuous $f : V \to W$ the mapping $\Phi(f)$ is a continuous equivariant mapping $(V, \alpha) \to (W, \beta)$.

  Proof. It is clear that $\Phi(f)$ is continuous. To show that $\Phi(f)$ is equivariant, note that
  \begin{center}
    \begin{tikzcd}
      V \arrow[r, "\alpha"] \arrow[d, "\alpha"'] \arrow[rrr, bend left, dotted,  "\Phi(f)"] & E(V) \arrow[d, "\delta_V"] \arrow[r, "E(f)"] & E(W) \arrow[d, "\delta_W"] \arrow[r, "\gamma"] & W \arrow[d, "\beta"]\\
      E(V) \arrow[r, "E(\alpha)"'] \arrow[rrr, bend right, dotted, "E(\Phi(f))"']              & EE(V) \arrow[r, "EE(f)"']                    & EE(W) \arrow[r, "E(\gamma)"'] & E(W)
    \end{tikzcd}
  \end{center}
  The leftmost square commutes because $\alpha$ is an $(E, \delta)$-module, the central square commutes because $\delta$ is natural,
  and the rightmost square commutes because $\gamma$ is a left $(E, \delta)$-inverse of $\beta$.~$\dashv$

  \bigskip

  Claim~2. For every continuous equivariant $\psi : (V, \alpha) \to (W, \beta)$ we have that $\Phi(\psi) = \psi$.

  Proof. Since $\psi$ is equivariant:
  \begin{center}
    \begin{tikzcd}
      V \arrow[r, "\psi"] \arrow[d, "\alpha"'] & W \arrow[d, "\beta"]\\
      E(V) \arrow[r, "E(\psi)"'] & E(W)
    \end{tikzcd}
  \end{center}
  This fact, together with $\gamma \circ \beta = \id_W$, reduces the proof to straightforward computation:
  $\Phi(\psi) = \gamma \circ E(\psi) \circ \alpha = \gamma \circ \beta \circ \psi = \psi$.~$\dashv$

  \bigskip

  Moving to $\Top$ for the second stage of the proof, let $X = U(V)$ and $Y = U(W)$.
  We shall implicitly use the convention that $U(f) = f$.
  In particular, for $f \in C(X, Y) = C(V, W)$ it makes sense to consider $\Phi(f)$; and
  if $(V, \alpha)$ is an $E$-coalgebra in $\FdnVect$ then $(X, \alpha)$ is an $E$-coalgebra in $\Top$
  (here, $E : \Top \to \Top$ is constructed as the left Kan extension of $UE : \FdnVect \to \Top$ along $U$,
  see the beginning of the section).

  Take a subcoalgebra $(K, \theta)$ of $(X, \alpha)$ such that $K$ is a compact subset of $X$.
  Since $(K, \theta)$ is a subcoalgebra of $(X, \alpha)$, the inclusion map $i : K \to X$, $i(x) = x$ is equivariant:
  \begin{center}
    \begin{tikzcd}
      K \arrow[r, "i"] \arrow[d, "\theta"'] & X \arrow[d, "\alpha"]\\
      E(K) \arrow[r, "E(i)"'] & E(X)
    \end{tikzcd}
  \end{center}
  Note that
  $$
    C(X, Y) \circ i = \{\phi \circ i : \phi \in C(X, Y)\} \subseteq C(K, Y).
  $$
  It is easy to check that $C(X, Y) \circ i$ is a linear
  subspace of $C(K, Y)$, and hence a normed space ($C(X, Y) \circ i$ inherits the sup-norm from $C(V, W)$
  restricted to $K$). Define $\Psi : C(X, Y) \circ i \to C(X, Y) \circ i$ by
  $$
    \Psi(\phi \circ i) = \Phi(\phi) \circ i.
  $$

  \bigskip

  Claim~3. The definition of $\Psi$ is correct.

  Proof. To show that the definition of $\Psi$ is correct we have to show that
  $$
    \phi_1 \circ i = \phi_2 \circ i \Rightarrow \Phi(\phi_1) \circ i = \Phi(\phi_2) \circ i.
  $$
  This follows by an easy computation:
  \begin{align*}
    \phi_1 \circ i &= \phi_2 \circ i\\
    \gamma \circ E(\phi_1) \circ E(i) \circ \theta &= \gamma \circ E(\phi_2) \circ E(i) \circ \theta\\
    \gamma \circ E(\phi_1) \circ \alpha \circ i &= \gamma \circ E(\phi_2) \circ \alpha \circ i\\
    \Phi(\phi_1) \circ i &= \Phi(\phi_2) \circ i.
  \end{align*}

  \bigskip

  So, $\Psi$ is a linear and continuous operator on the normed space $C(X, Y) \circ i$.

  \bigskip

  Claim~4. If $\Phi(\phi) = \phi$ then $\Psi(\phi \circ i) = \phi \circ i$.

  Proof. $\Psi(\phi \circ i) = \Phi(\phi) \circ i = \phi \circ i$.

  \bigskip

  Since $\Psi$ is a continuous (and hence bounded) linear operator on the normed vector space $C(X, Y) \circ i$
  it has a finite norm $\|\Psi\| \in \RR$.

  Now, take any continuous equivariant $\phi : (X, \alpha) \to (Y, \beta)$ and $\epsilon > 0$.
  By the Universal Approximation Theorem \ref{euat.thm.UAT-main} there exists an
  $$
    f = Q \circ \bar\sigma \circ P \in \mathcal{NN}^{\sigma}
  $$
  such that $\|\phi(x) - f(x)\| < \dfrac{\epsilon}{2\|\Psi\|}$ for all $x \in K$.
  Hence,
  $$
    \|\phi \circ i - f \circ i\| \le \dfrac{\epsilon}{2\|\Psi\|}
  $$
  in $C(X, Y) \circ i$. Put $\ell = \Phi(f)$. Then:
  $$
    \ell = \gamma \circ E(Q) \circ E(\bar\sigma) \circ E(P) \circ \alpha.
  $$
  Note that $\ell : (X, \alpha) \to (Y, \beta)$ is equivariant because of Claim~1,
  while Claim~2 yields:
  \begin{align*}
    \|\phi \circ i - \ell \circ i\|
    & = \|\Psi(\phi \circ i) - \Psi(f \circ i)\|\\
    &\le \|\Psi\| \cdot \|\phi \circ i - f \circ i\|\\
    &\le \|\Psi\| \cdot \dfrac{\epsilon}{2\|\Psi\|} < \epsilon.
  \end{align*}
  In other words,
  $$
    \|\phi(x) - \ell(x)\| < \epsilon, \text{ for all } x \in K.
  $$
  
  \bigskip

  To complete the proof we still have to show that $\ell \in \mathcal{VNN}^{E(\sigma)}$.
  Moving back to $\FdnVect$ for the final stage of the proof,
  recall that the category of vector spaces is an abelian category, so each finite power
  $\RR^d$ is also a copower $\RR \oplus \ldots \oplus \RR$ ($d$ times). Consequently,
  each projection $\pi_j : \RR^d \to \RR : (x_1, \ldots, x_d) \mapsto x_j$,
  $1 \le j \le d$, is accompanied by an injection
  $\iota_j : \RR \to \RR^d : x \mapsto (0, \ldots, 0, x, 0 \ldots, 0)$ with $x$ at the $j$th place.
  Note that $\pi_j \circ \iota_j = \id_\RR$, $1 \le j \le d$, and that
  $$
    [\iota_1, \ldots, \iota_d] \circ \langle \pi_1, \ldots, \pi_d \rangle = \id_{\RR^d}.
  $$
  Here, $[f_1, \ldots, f_d] : \RR^d \to \RR^d$ denotes the cotupling:
  $$
    [f_1, \ldots, f_d](x_1, \ldots, x_d) = f_1(x_1) + \ldots + f_d(x_d),
  $$
  and $\langle f_1, \ldots, f_d \rangle : \RR^d \to \RR^d$ denotes the tupling:
  $$
    \langle f_1, \ldots, f_d \rangle(\mathbf{x}) = (f_1(\mathbf{x}), \ldots, f_d(\mathbf{x})), \quad \mathbf{x} \in \RR^d.
  $$
  Note that both the tupling and the cotupling of linear maps is a linear map.

  Since $\bar\sigma : \RR^d \to \RR^d$ applies $\sigma$ coordinatewise, the following diagram commutes for all $1 \le j \le d$:
  \begin{center}
    \begin{tikzcd}
      \RR^d \arrow[r, "\bar\sigma"] \arrow[d, "\pi_j"'] & \RR^d \arrow[d, "\pi_j"] \\
      \RR \arrow[r, "\sigma"'] & \RR
    \end{tikzcd}
  \end{center}
  so after applying $E$ to each of the diagrams and tupling the arrows $E(\pi_j)$ we get:
  \begin{center}
    \begin{tikzcd}
      E(\RR^d) \arrow[rr, "E(\bar\sigma)"] \arrow[d, "{\langle E(\pi_1), \ldots, E(\pi_d)\rangle}"'] & & E(\RR^d) \arrow[d, "{\langle E(\pi_1), \ldots, E(\pi_d)\rangle}"] \\
      E(\RR)^n \arrow[rr, "E(\sigma) \times \ldots \times E(\sigma)"'] & & E(\RR)^n
    \end{tikzcd}
  \end{center}
  hence:
  $$
    \langle E(\pi_1), \ldots, E(\pi_d)\rangle \circ E(\bar\sigma) = (E(\sigma) \times \ldots \times E(\sigma)) \circ \langle E(\pi_1), \ldots, E(\pi_d)\rangle.
  $$
  An easy calculation together with the fact that $E$ is linear on morphisms suffices to
  show that $[E(\iota_1), \ldots, E(\iota_d)]$ is a left inverse of $\langle E(\pi_1), \ldots, E(\pi_d) \rangle$:
  \begin{multline*}
    [E(\iota_1), \ldots, E(\iota_d)] \circ \langle E(\pi_1), \ldots, E(\pi_d) \rangle (\mathbf x) =\\
    = E(\iota_1) \circ E(\pi_1) (\mathbf x) + \ldots + E(\iota_d) \circ E(\pi_d) (\mathbf x) =\\
    = E(\iota_1 \circ \pi_1) (\mathbf x) + \ldots + E(\iota_d \circ \pi_d) (\mathbf x) =\\
    = E(\iota_1 \circ \pi_1 + \ldots + \iota_d \circ \pi_d) (\mathbf x) = E(\id_{\RR^d}) (\mathbf x) = \mathbf x.
  \end{multline*}
  Therefore,
  $$
    E(\bar\sigma) = [E(\iota_1), \ldots, E(\iota_d)] \circ (E(\sigma) \times \ldots \times E(\sigma)) \circ \langle E(\pi_1), \ldots, E(\pi_d)\rangle,
  $$
  which allows us to write $\ell$ as:
  \begin{multline*}
    \ell = \gamma \circ E(Q) \circ [E(\iota_1), \ldots, E(\iota_d)] \circ\\
    \circ (E(\sigma) \times \ldots \times E(\sigma)) \circ \\
    \circ \langle E(\pi_1), \ldots, E(\pi_d)\rangle \circ E(P) \circ \alpha.
  \end{multline*}
  In other words,
  $$
    \ell = Q' \circ (E(\sigma) \times \ldots \times E(\sigma)) \circ P',
  $$
  where $Q' = \gamma \circ E(Q) \circ [E(\iota_1), \ldots, E(\iota_d)]$ and $P' = \langle E(\pi_1), \ldots, E(\pi_d)\rangle \circ E(P) \circ \alpha$
  are linear maps. This concludes the proof that $\ell \in \mathcal{VNN}^{E(\sigma)}$.
\end{proof}

Let us conclude the section with a remark on compact subcoalgebras of $(X, \alpha)$ for $X = U(V)$.
Let $S$ be a finite sample space understood as a discrete topological space.
A feature extractions procedure yields an injective function $e : S \to U(V)$ for some normed finitely dimensional vector space $V$.
If the sample space exhibits the specified type of symmetries, this can be modeled by a coalgebra structure $\theta : S \to E(S)$.
Note that $\theta$ is trivially continuous, so $\theta : S \to E(S)$ lives in $\Top$.
If $e$ is an equivariant embedding of $(S, \theta)$ into $(U(V), \alpha)$ then $e$ is a coalgebra homomorphism:
\begin{center}
  \begin{tikzcd}
    S \arrow[r, "e"] \arrow[d, "\theta"'] & U(V) \arrow[d, "\alpha"]\\
    E(S) \arrow[r, "E(e)"'] & EU(V)
  \end{tikzcd}
\end{center}
where we implicitly use that $U(\alpha) = \alpha$ and $UE = EU$. So, $S$, being finite, is a compact subcoalgebra of $(U(V), \alpha)$.


\begin{thebibliography}{99}
\bibitem{AHS}
  J. Ad\'amek, H. Herrlich, G. E. Strecker:
  \textit{Abstract and Concrete Categories: The Joy of Cats},
  Dover Books on Mathematics, Dover Publications 2009.

\bibitem{Bar93}
  M. Barr.
  Terminal coalgebras in well-founded set theory.
  Theoretical Computer Science 114(1993), 299--315.

\bibitem{GDL-main}
  M. M. Bronstein, J. Bruna, Y. LeCun, A. Szlam, P. Vandergheynst.
  Geometric deep learning: going beyond Euclidean data.
  IEEE Signal Processing Magazine, 34(2017), 18--42.

\bibitem{GDL-Protobook}
  M. M. Bronstein, J. Bruna, T. Cohen, P. Veli\v ckovi\'c.
  Geometric Deep Learning: Grids, Groups, Graphs, Geodesics, and Gauges.
  Preprint available as arXiv:2104.13478

\bibitem{Cat-GBL-2022}
  G. S. H. Cruttwell, B. Gavranovi{\'{c}}, N. Ghani, P. Wilson, F. Zanasi.
  Categorical Foundations of Gradient-Based Learning.
  In: I. Sergey (Ed), Programming Languages and Systems.
  Springer International Publishing 2022, 
  ESOP 2022. Lecture Notes in Computer Science, vol 13240. Springer, Cham.
  pp 1--28.

\bibitem{DengEtAl-VectorNeurons2021}
  C. Deng, O. Litany, Y. Duan, A. Poulenard, A. Tagliasacchi and L. Guibas.
  Vector Neurons: A General Framework for SO(3)-Equivariant Networks.
  2021 IEEE/CVF International Conference on Computer Vision (ICCV), Montreal, QC, Canada, 2021,
  pp.\ 12180--12189. doi: 10.1109/ICCV48922.2021.01198  

\bibitem{backprop-as-ftr}
  B. Fong, D. Spivak, R. Tuy\'{e}ras.
  Backprop as functor: a compositional perspective on supervised learning.
  In: Proceedings of the 34th Annual ACM/IEEE Symposium on Logic in Computer Science (LICS '19), IEEE Press, Vancouver, Canada, 2021, Article No.: 11
  (13 pages)

\bibitem{gavran-2019}
  B. Gavranovi\'c.
  Learning Functors using Gradient Descent.
  In: Proceedings of Applied Category Theory 2019, Electronic Proceedings in Theoretical Computer Science 323, 230--245.

\bibitem{CDL-position-paper}
  B. Gavranovi\'c, P. Lessard, A. Dudzik, T. von Glehn, J. G. M. Ara\'ujo, P. Veli\v ckovi\'c.
  Position: Categorical Deep Learning is an Algebraic Theory of All Architectures.
  Preprint available as arXiv:2402.15332

\bibitem{KatzirEtAl-VectorNeurons-2022}
  Katzir, O., Lischinski, D., Cohen-Or, D.:
  \emph{Shape-Pose Disentanglement Using SE(3)-Equivariant Vector Neurons}.
  In: Avidan, S., Brostow, G., Ciss\'e, M., Farinella, G.M., Hassner, T. (eds) Computer Vision – ECCV 2022. ECCV 2022.
  Lecture Notes in Computer Science, vol 13663 (2022).
  Springer, Cham. doi: 10.1007/978-3-031-20062-5\_27

\bibitem{Pinkus-1999}
  A. Pinkus.
  Approximation theory of the mlp model in neural networks.
  Acta Numerica, 8(1999), 143--195.

\bibitem{Riehl}
  Emily Riehl.
  Category Theory in Context.
  Dover Publications, 2017. Publicly available at:
  https://emilyriehl.github.io/files/context.pdf

\bibitem{rutten-2000}
  J. J. M. M. Rutten.
  Universal coalgebra: a theory of systems.
  Theoretical Computer Science 249 (2000), 3--80.

\bibitem{Yarotsky-2022}
  D. Yarotsky.
  Universal approximations of invariant maps by neural networks.
  Constructive Approximation, 55(1) (2022), 407--474.
  doi: 10.1007/s00365-021-09546-1. 

\bibitem{CDL-website}
  https://categoricaldeeplearning.com/

\bibitem{GDL-website}
  https://geometricdeeplearning.com/

\end{thebibliography}
\end{document}